\documentclass[twoside,11pt]{article}

\usepackage{jmlr2e}
\usepackage{amsmath,amssymb,amsfonts}
\usepackage{subfig}
\usepackage{booktabs}
\usepackage{microtype}
\usepackage{url}
\usepackage{lastpage}

\jmlrheading{0}{2026}{1-\pageref{LastPage}}{1/26}{}{ }{Pal et al.}

\ShortHeadings{Blind-Spot Mass for Deployment Coverage}{Pal, Bhattacharya, and Singh}
\firstpageno{1}

\title{Blind-Spot Mass: A Good--Turing Framework for Quantifying Deployment Coverage Risk in Machine Learning Systems}

\author{\name Biplab Pal\thanks{Corresponding author.} \email bpal1@umbc.edu\\
       \addr CARDS (Center for Real-Time Distributed Sensing and Autonomy),\\ University of Maryland, Baltimore County (UMBC)\\
       \AND
       \name Santanu Bhattacharya\\
       \addr Massachusetts Institute of Technology\\
       \AND
       \name Madanjit Singh\\
       \addr Ambient Scientific Inc.}

\editor{(to be assigned)}

\begin{document}

\maketitle

\begin{abstract}
Blind-spot mass is a Good--Turing framework for quantifying deployment coverage risk in machine learning. 
In modern ML systems, operational state distributions are often heavy-tailed, implying that a long tail of valid but rare states is structurally under-supported in finite training and evaluation data. 
This creates a form of ``coverage blindness'': models can appear accurate on standard test sets yet remain unreliable across large regions of the deployment state space.

We propose blind-spot mass $B_n(\tau)$, a deployment metric estimating the total probability mass assigned to states whose empirical support falls below a threshold $\tau$. 
$B_n(\tau)$ is computed using Good--Turing unseen-species estimation and yields a principled estimate of how much of the operational distribution lies in reliability-critical, under-supported regimes. 
We further derive a coverage-imposed accuracy ceiling, decomposing overall performance into supported and blind components and separating capacity limits from data limits.

We validate the framework in wearable human activity recognition (HAR) using wrist-worn inertial data. 
Under a deployment-refined operational abstraction, $\widehat{B}_n(\tau)$ indicates that $\approx 95\%$ of deployment probability mass sits below $\tau{=}5$, implying that most valid regimes are effectively ``unseen'' by training support. 
We then replicate the same analysis in the MIMIC-IV hospital database (275 admissions), where the blind-spot mass curve converges to the same $\approx 95\%$ at $\tau{=}5$ across clinical state abstractions. 
This replication across structurally independent domains --- differing in modality, feature space, label space, and application --- shows that blind-spot mass is a general ML methodology for quantifying combinatorial coverage risk, not an application-specific artifact.

Blind-spot decomposition identifies which activities or clinical regimes dominate risk, providing actionable guidance for industrial practitioners on targeted data collection, normalization/renormalization, and physics- or domain-informed constraints for safer deployment.
\end{abstract}


\section{Introduction}
The reliability of a machine learning system in deployment depends not only on its accuracy on a held-out test set, but on the relationship between its training distribution and the full range of states it will encounter in operation.
Real-world operational distributions are often heavy-tailed: most probability mass concentrates in common states, while a long tail of rare but valid states receives little or no representation in finite training datasets.
A model can therefore achieve high benchmark accuracy yet behave unreliably across a non-trivial fraction of its operational distribution---not necessarily because of distribution shift, but because those regions were structurally invisible during training.

\begin{figure}[t]
  \centering
  \includegraphics[width=0.95\linewidth]{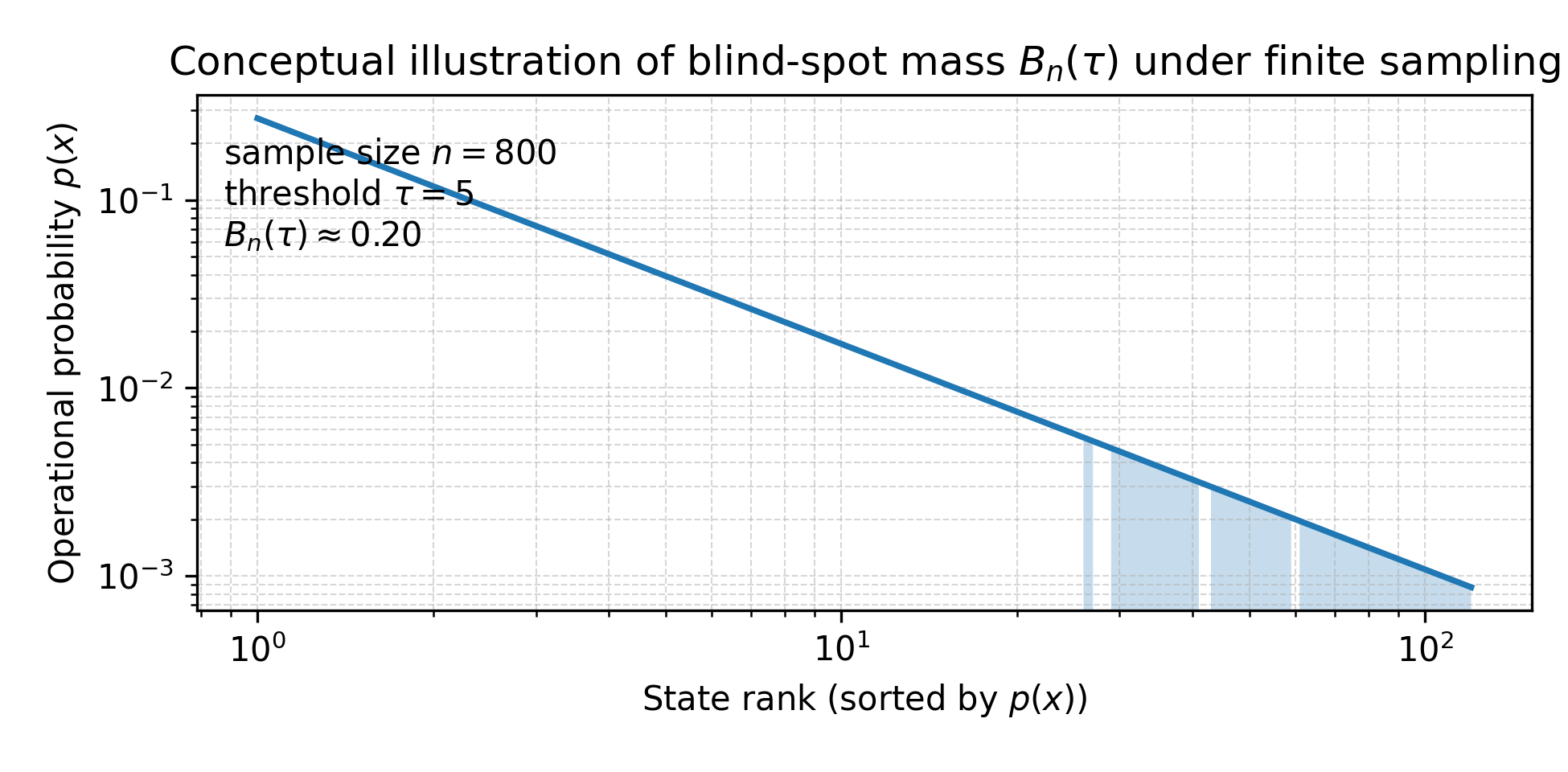}
  \caption{Conceptual illustration of blind-spot mass under finite sampling. States are ordered by operational probability $p(x)$ (heavy-tailed). A finite sample of size $n$ yields empirical supports $N_n(x)$; blind-spot mass $B_n(\tau)$ corresponds to the total probability of states whose support falls below a threshold $\tau$.}
  \label{fig:concept}
\end{figure}

Existing approaches address adjacent deployment questions but do not answer a fundamental coverage question: \emph{What fraction of the deployment distribution lies in regions where the model has insufficient data support to be reliable?}
Out-of-distribution (OOD) detection assigns anomaly scores to individual test samples at inference time; conformal prediction provides coverage guarantees under calibration assumptions; and shift detection monitors distributional changes during operation.
None of these methods directly quantifies the aggregate probability mass in under-supported regions \emph{before} deployment.

We bridge this gap by adapting Good--Turing unseen-event estimation to deployment coverage.
Good--Turing theory was developed to estimate the probability of encountering unseen categories from finite samples; in modern ML deployments, training and calibration data are likewise finite samples from an operational distribution whose tail is structurally underrepresented.
We formalize this intuition as \emph{blind-spot mass} $B_n(\tau)$, which estimates the probability mass of operational states whose empirical support is below a threshold $\tau$.

Although the framework is general, we focus our empirical validation on wearable human activity recognition (HAR), a domain where industrial deployments frequently face combinatorial variation (body size, gait style, placement/orientation drift, clothing, and environment) and where edge models must remain compact.
Our contributions are:
\begin{itemize}
\item We define blind-spot mass $B_n(\tau)$ as a thresholded Good--Turing coverage metric over an operational state space, and provide estimators based on frequency-of-frequencies.
\item We derive a coverage-imposed accuracy ceiling that decomposes total accuracy into supported and blind regions, isolating limits due to data support from limits due to model capacity (Proposition~\ref{prop:ceiling}).
\item We instantiate the framework in wearable HAR and validate it on an in-house dataset and the PAMAP2 benchmark, showing that state refinement induces sharp increases in blind-spot mass and identifying dominant contributors via blindness decomposition (Figs.~\ref{fig:fig1}--\ref{fig:fig5} and Table~\ref{tab:risk_blindness}).
\end{itemize}
Sections~\ref{sec:framework}--\ref{sec:results} present the framework, validation setup, and empirical results; Section~\ref{sec:discussion} discusses deployment implications for industrial practice.
\section{Related Work}
\label{sec:related}
Blind-spot mass is a \emph{coverage quantification} tool that complements (rather than replaces) uncertainty estimation and shift monitoring.
We position the contribution relative to four adjacent literatures that are commonly used to reason about deployment risk.

\subsection{Out-of-distribution and novelty detection}
OOD detection methods aim to identify inputs that are unlikely under the training distribution, typically by assigning an anomaly score to each test sample at inference time.
Common baselines include maximum softmax probability and related confidence heuristics~\citep{hendrycks2017baseline}, while more recent approaches use energy scores~\citep{liu2020energy} or class-conditional feature-space distances such as Mahalanobis statistics~\citep{lee2018mahalanobis}.
These methods are valuable for \emph{flagging} anomalous samples, but they do not quantify how much of the \emph{deployment distribution} lies in low-support regimes under a specified support requirement.

The key distinction is the unit of analysis and timing.
While OOD detection methods assign anomaly scores to individual test samples at inference time, blind-spot mass operates at the level of the deployment distribution, quantifying the aggregate probability mass residing in under-supported regions \emph{prior} to deployment.
As a result, a system can exhibit low OOD rate yet still suffer severe long-tail under-support within the in-distribution manifold.

Beyond this conceptual distinction, many OOD benchmarks operationalize ``outliers'' via curated auxiliary datasets.
Such setups are valuable for stress testing but do not directly address the common industrial failure mode where rare \emph{in-distribution} states are valid yet under-sampled (for example, unusual sensor placements or transitional motion regimes).
Because blind-spot mass is computed from the empirical support structure of the training/calibration sample itself, it can surface long-tail under-support even when no explicit outlier dataset exists and even when per-sample anomaly detectors appear well calibrated on standard test splits.

\subsection{Conformal prediction and abstention}
Conformal prediction provides set-valued predictors with finite-sample marginal coverage guarantees under exchangeability, often using a held-out calibration set~\citep{angelopoulos2021conformal}.
In practice, conformal methods are frequently paired with selective prediction or abstention to trade coverage for risk.
However, conformal guarantees are conditional on the calibration set being sufficiently representative of the operational distribution.

Blind-spot mass targets a different question: whether the available data provide sufficient \emph{support} across the operational state space at a given reliability threshold.
In this sense, $B_n(\tau)$ is an upstream diagnostic that can inform whether a desired operating point is feasible without collecting additional targeted data or reducing the effective state space through normalization or physics-informed constraints.

This relationship makes the methods complementary in practice.
For example, when $B_n(\tau)$ remains large at the support level required for a desired risk bound, conformal sets may achieve nominal marginal coverage while still concentrating errors in low-support regions, yielding poor utility for selective prediction.
Running $B_n(\tau)$ alongside conformal calibration provides a principled way to determine whether additional targeted calibration data or a coarser operational state abstraction is needed before conformal guarantees become operationally meaningful.

\subsection{Dataset shift detection and monitoring}
Shift detection and monitoring aim to determine whether deployment data deviate from training data, using two-sample tests or black-box indicators in learned feature spaces~\citep{rabanser2019failing}.
Such monitoring is essential for continuous evaluation once a system is deployed.
Yet shift tests are not designed to answer coverage sufficiency questions such as: what fraction of deployment lies below $\tau$-support, or which regions dominate that fraction?

Blind-spot mass focuses directly on the support structure of the operational state distribution and is therefore complementary to shift detection.
A practical workflow is to use $B_n(\tau)$ as a pre-deployment coverage diagnostic and to use shift monitoring to detect changes in the operational distribution post-deployment.

A further distinction is that shift detection fundamentally requires a sample of deployment observations, whereas blind-spot mass can be computed entirely pre-deployment.
Moreover, a system can suffer large blind-spot mass even when there is \emph{no} distribution shift in the usual sense (the operational distribution matches training in expectation but exhibits a heavy tail that is poorly covered at finite $n$).
In such cases, shift tests may correctly report ``no shift'' while deployment reliability is still dominated by under-supported states; $B_n(\tau)$ is designed to diagnose precisely this regime.

\subsection{Good--Turing estimation and unseen-event probability}
Blind-spot mass builds on classical work on unseen-event probability and species estimation.
Good--Turing estimation~\citep{good1953} and related smoothing methods~\citep{church1991} estimate the total probability of events not observed in a sample (the \emph{unseen mass}).
In ecology and biodiversity, unseen-species estimators such as Chao's lower bound~\citep{chao1984} quantify missing richness under finite sampling, and modern treatments study optimal estimation in large alphabets~\citep{orlitsky2016,efron1976}.

Our contribution is to generalize unseen-mass reasoning to a deployment-facing, thresholded notion $B_n(\tau)$ that captures reliability-relevant support requirements and to connect $B_n(\tau)$ to an explicit accuracy decomposition and actionable state-refinement analyses in edge sensing.

While unseen-event estimation has been widely studied in statistics and language modeling, it has been less explicitly connected to deployment risk assessment in machine learning systems.
A key step in this translation is defining an operational state abstraction---the analogue of ``species''---that captures reliability-relevant variation and can be refined (or constrained) to trade expressivity for support.
The thresholded form $B_n(\tau)$ then links classical unseen-mass reasoning to actionable engineering questions about minimum per-state support and abstention activation, rather than solely estimating the probability of entirely unseen categories.

\section{Blind-Spot Mass: Mathematical Framework}\label{sec:framework}
We formalize deployment coverage risk through an operational state space $\Omega$ and a thresholded notion of under-support.
The definitions are model-agnostic: they describe limitations imposed by data coverage rather than by a particular classifier.

\subsection{Operational state space and combinatorial growth}
Let $\Omega$ denote the operational state space; each $x\in\Omega$ represents a configuration of latent factors relevant to model behavior.
If $d$ condition factors are modeled as $\Omega_1,\ldots,\Omega_d$, then
\begin{equation}
\Omega = \prod_{i=1}^{d} \Omega_i, \qquad |\Omega| = \prod_{i=1}^{d} |\Omega_i|.
\label{eq:omega_factor}
\end{equation}
Even moderate $|\Omega_i|$ yields exponential growth, motivating abstractions and invariances.

\subsection{Effective state space}
Physical constraints, invariances, and modeling assumptions reduce the number of distinguishable behaviors.
We define an \emph{effective} state-space size
\begin{equation}
K_{\mathrm{eff}} \le |\Omega|.
\label{eq:keff}
\end{equation}
In wearables, $K_{\mathrm{eff}}$ can be reduced by renormalization (e.g., orientation handling, gravity separation, cadence normalization) and by physics-informed structure.

\subsection{Blind-spot mass (coverage blindness)}
Let $X_1,\ldots,X_n\sim P$ be i.i.d.\ samples over $\Omega$ with counts $N_n(x) = \sum_{i=1}^n \mathbf{1}\{X_i=x\}$.
For a support threshold $\tau\ge 1$, define the blind-spot mass
\begin{equation}
B_n(\tau)=\sum_{x\in\Omega} P(x)\,\mathbf{1}\{N_n(x)<\tau\}.
\label{eq:bn}
\end{equation}
$B_n(\tau)$ quantifies the fraction of deployment probability mass that is under-supported at level $\tau$, independent of model class.
Coarsening the state space (merging states) can only increase support counts and therefore yields a \emph{lower bound} on blindness for the true (finer) operational state space.

\subsection{Risk-weighted blindness}
To encode heterogeneous consequences, define a risk function $w:\Omega\to\mathbb{R}_+$ and the risk-weighted blind-spot mass
\begin{equation}
B_n^*(\tau)=\sum_{x\in\Omega} P(x)\,w(x)\,\mathbf{1}\{N_n(x)<\tau\}.
\label{eq:bnstar}
\end{equation}

\subsection{Unseen mass and Good--Turing}
For $\tau=1$, $B_n(1)$ is the unseen probability mass.
Under mild assumptions, Good--Turing estimation yields $\widehat{B}_n(1)\approx f_1/n$, where $f_1$ is the number of singleton states \citep{good1953,efron1976,orlitsky2016}:
\begin{equation}
\widehat{B}_n(1)\approx \frac{f_1}{n}.
\label{eq:gt}
\end{equation}

\subsection{Wearable operational state abstraction}
For wearable inertial HAR, we instantiate $x\in\Omega$ as a tuple of condition factors
\begin{equation}
x=(a,u,p,e,c),
\label{eq:wearable_state}
\end{equation}
where $a$ denotes the activity label; $u$ captures user attributes (e.g., body size, age band);
$p$ denotes sensor placement and orientation; $e$ denotes environment (e.g., floor compliance, incline);
and $c$ denotes interaction context (e.g., chair/table, transitions).
The true $\Omega$ is high-dimensional, so empirical analyses often use coarser partitions.

\subsection{Coverage-imposed accuracy ceiling}
Blind-spot mass induces a simple performance decomposition that separates supported and blind regions.

\begin{proposition}[Coverage-imposed ceiling]\label{prop:ceiling}
Let $S_\tau=\{x: N_n(x)\ge \tau\}$ be supported states and $U_\tau=\Omega\setminus S_\tau$ blind states.
For any predictor, the overall accuracy satisfies
\begin{equation}
\mathrm{Acc} = (1-B_n(\tau))\,\mathrm{Acc}_{\mathrm{sup}} + B_n(\tau)\,\mathrm{Acc}_{\mathrm{blind}},
\label{eq:mixture}
\end{equation}
where $\mathrm{Acc}_{\mathrm{sup}}=\mathbb{P}(\hat{Y}=Y\mid X\in S_\tau)$ and $\mathrm{Acc}_{\mathrm{blind}}=\mathbb{P}(\hat{Y}=Y\mid X\in U_\tau)$.
\end{proposition}

\begin{proof}
By the law of total probability,
\begin{align*}
\mathrm{Acc} &= \mathbb{P}(\hat{Y}=Y)\\
&= \mathbb{P}(\hat{Y}=Y\mid X\in S_\tau)\,\mathbb{P}(X\in S_\tau)
 + \mathbb{P}(\hat{Y}=Y\mid X\in U_\tau)\,\mathbb{P}(X\in U_\tau).
\end{align*}
Noting $\mathbb{P}(X\in U_\tau)=B_n(\tau)$ from~\eqref{eq:bn} and $\mathbb{P}(X\in S_\tau)=1-B_n(\tau)$ proves~\eqref{eq:mixture}.
\end{proof}

\section{Validation Domains and Datasets}\label{sec:setup}
We validate blind-spot mass on wearable inertial HAR using two datasets: an in-house dataset collected in collaboration with an industry partner and the open PAMAP2 benchmark.

\subsection{Datasets}
\textbf{In-house dataset (Figs.~\ref{fig:fig1}a--\ref{fig:fig5}a).}
The in-house dataset contains 12 activity classes and $n=1674$ windows.
A compact edge model was evaluated across three wearable form factors.
Each configuration used 55 test windows (5 per class), making per-class accuracy estimates statistically fragile; we therefore report 95\% Wilson confidence intervals in Fig.~\ref{fig:fig2}a.

\textbf{PAMAP2 benchmark (Figs.~\ref{fig:fig1}b--\ref{fig:fig5}b).}
PAMAP2 was collected with three IMUs (hand, chest, ankle) at 100~Hz and includes 18 activities in total \citep{reiss2012iswc,reiss2012abra}.
We process two subjects (101 and 105) from the raw \texttt{.dat} files and discard activityID$=0$ (transients) per the PAMAP2 documentation.
This subset contains 14 labeled activities and $n=1533$ windows using 5~s windows with 2.5~s stride.

\textbf{MIMIC-IV Demo (cross-domain replication; Fig.~\ref{fig:crossdomain}).}
To validate that coverage blindness is not a wearable-specific artifact, we also analyse the MIMIC-IV Clinical Database Demo v2.2 \citep{johnson2023mimiciv}. 
This public demo subset contains 275 hospital admissions.
For each admission, we construct a simple discrete clinical state abstraction using the 4-digit prefix of the primary ICD diagnosis code (sequence number 1 in \texttt{diagnoses\_icd}), yielding a heavy-tailed state distribution suitable for blind-spot analysis.

\subsection{Ethics approval and consent}
The in-house dataset consists of fully anonymized inertial sensor data (accelerometer and gyroscope readings) collected by a commercial partner during wearable device validation.
The dataset contains no personally identifiable information, demographic data, or linking metadata that could be used to identify individual participants.
Participants were assigned random pseudonymous labels for data management purposes, and the data collection entity retained no linking key between pseudonyms and participant identities.
All participants provided informed consent for their anonymized sensor data to be used for research purposes.
All methods were carried out in accordance with relevant data protection regulations and ethical guidelines.
This secondary analysis of fully de-identified sensor data does not constitute human subjects research under 45 CFR 46.102(e)(1) and did not require institutional review board approval.
The PAMAP2 dataset was collected under ethics approval as documented in the original study \citep{reiss2012iswc,reiss2012abra}.

\subsection{Models and evaluation protocol}
The classifier differs between datasets (tiny CNN vs.\ logistic regression) by design: blind-spot mass evaluates \emph{distributional support} and state coverage rather than classifier capacity.
Any competent classifier operating under the same support constraints obeys Proposition~\ref{prop:ceiling}; the figures are intended to expose the coverage bottleneck, not to optimize architecture.

For PAMAP2 accuracy (Fig.~\ref{fig:fig2}b), we use a lightweight multinomial logistic regression classifier with window-level summary features (mean and standard deviation of 3-axis accelerometer and gyroscope plus magnitude statistics; 16 features per sensor placement).
We use a stratified 70/30 train/test split to obtain per-class accuracies and Wilson 95\% confidence intervals.

\subsection{State refinement proxies and computation of blind-spot mass}
All blind-spot results are computed from the definition in~\eqref{eq:bn} using plug-in probabilities $\hat{P}(x)=N_n(x)/n$ at the chosen operational abstraction.
We evaluate how blind-spot mass changes when the state definition is refined from coarse activity-only states to composite states that include measurable proxies for placement/orientation and intensity.

\paragraph{PAMAP2 proxy definitions.}
To operationalize state refinement on PAMAP2, we introduce reproducible proxies computed per 5~s window from the chest IMU.
Let a window contain $L$ samples of chest accelerometer $\mathbf{a}_t\in\mathbb{R}^3$ and gyroscope $\boldsymbol{\omega}_t\in\mathbb{R}^3$.

\textbf{Tilt proxy and tilt-bin index.}
Define the window-mean acceleration vector
\begin{equation}
\boldsymbol{\mu}^{(\mathrm{acc})}=\frac{1}{L}\sum_{t=1}^{L} \mathbf{a}_t,
\qquad
\tilde{\boldsymbol{\mu}}^{(\mathrm{acc})}=\frac{\boldsymbol{\mu}^{(\mathrm{acc})}}{\lVert\boldsymbol{\mu}^{(\mathrm{acc})}\rVert_2}.
\end{equation}
We compute a tilt angle (relative to the sensor $z$-axis)
\begin{equation}
\phi=\arccos\big(|\tilde{\mu}^{(\mathrm{acc})}_z|\big)\in[0,\pi/2],
\end{equation}
and discretize it into $P$ bins (we use $P{=}6$) to obtain the placement/orientation proxy $p$
\begin{equation}
p = \min\Big(P{-}1,\ \big\lfloor P\,\phi/(\pi/2)\big\rfloor\Big)\in\{0,1,\ldots,P{-}1\}.
\label{eq:tilt_bin}
\end{equation}

\textbf{Intensity proxy and energy-bin.}
Define a windowed gyroscope energy
\begin{equation}
E=\frac{1}{L}\sum_{t=1}^{L}\lVert\boldsymbol{\omega}_t\rVert_2^2.
\label{eq:gyro_energy}
\end{equation}
We discretize $E$ into $Q$ bins (we use $Q{=}3$) using dataset quantiles $(q_0,q_{1/3},q_{2/3},q_1)$:
\begin{equation}
e = \begin{cases}
0, & q_0\le E\le q_{1/3},\\
1, & q_{1/3}<E\le q_{2/3},\\
2, & q_{2/3}<E\le q_1.
\end{cases}
\label{eq:energy_bin}
\end{equation}
These proxies are not claims of optimal state markers; they are operational, reproducible surrogates for placement/orientation variation and movement intensity that can be extracted consistently across public IMU datasets.

\paragraph{Operational state refinement on PAMAP2.}
Refining the abstraction from $x=a$ to $x=(a,p)$ and $x=(a,p,e)$ increases the observed number of distinguishable states from $K_{\mathrm{eff}}{=}14$ to $K_{\mathrm{eff}}{=}44$ to $K_{\mathrm{eff}}{=}78$ (subjects 101+105).
As predicted by~\eqref{eq:bn}, this refinement increases $\widehat{B}_n(\tau)$ for fixed $n$ (Fig.~\ref{fig:fig3}b).
\paragraph{Deployment-refined abstraction for cross-domain replication.}
In addition to the coarse-to-moderate refinements above (used in Figs.~\ref{fig:fig1}--\ref{fig:fig5}), we evaluate a more deployment-refined abstraction in which tilt/orientation is discretized at higher resolution and intensity is represented by finer quantization of window energy and mean angular-rate magnitude (Methods; used only for Fig.~\ref{fig:crossdomain}). 
This is intended to emulate the combinatorial granularity encountered in unconstrained field use, where small changes in orientation and motion intensity can define distinct operational regimes.

\section{Empirical Results and Cross-Domain Replication}\label{sec:results}
We analysed two wearable HAR datasets: (a) an in-house dataset collected under controlled but heterogeneous conditions and (b) the open PAMAP2 benchmark (subjects 101 and 105).

\subsection{Activity-level support was imbalanced}
Activity-level window counts varied substantially across activities in both datasets (Fig.~\ref{fig:fig1}).

\begin{figure}[t]
\centering
\subfloat[\textbf{In-house HAR:} activity-level window coverage.]{%
  \includegraphics[width=0.49\linewidth]{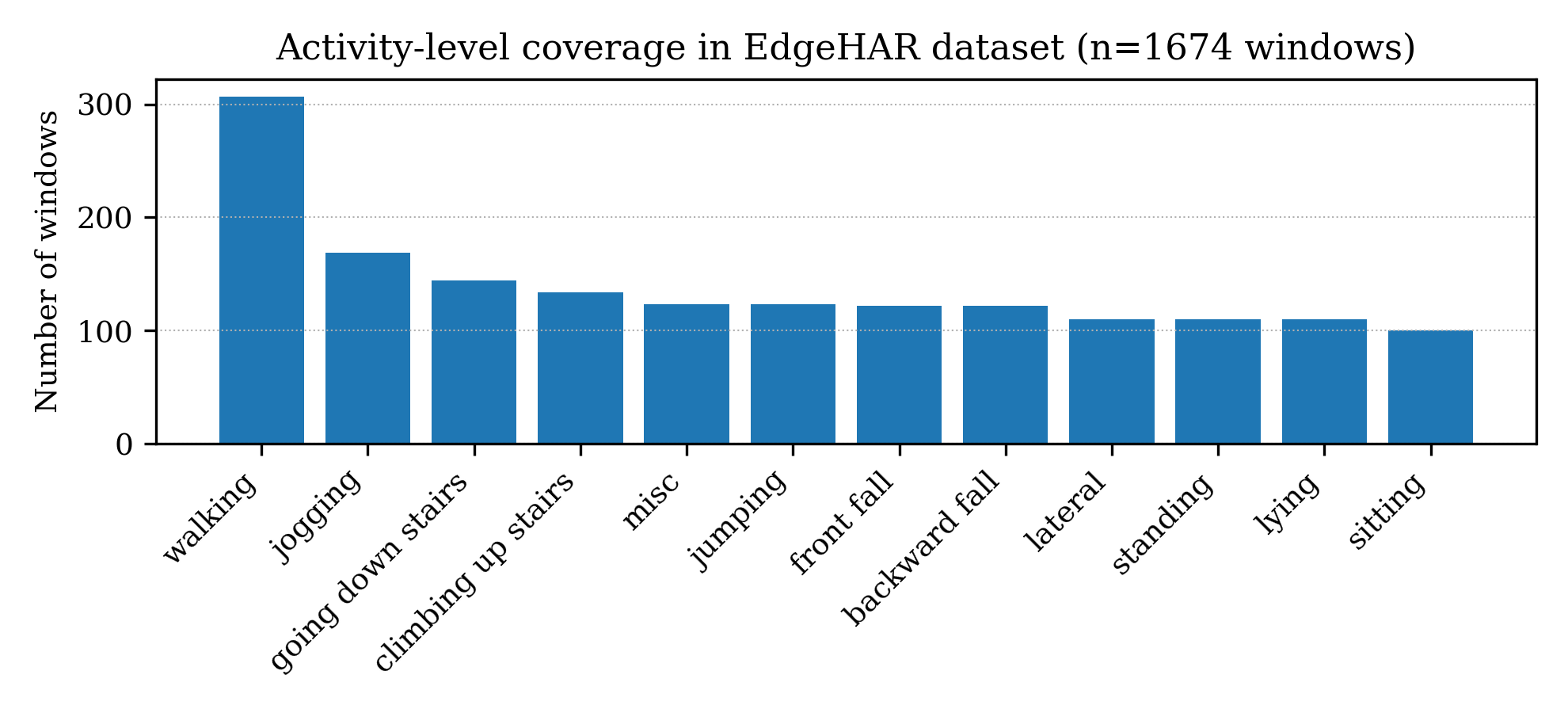}}%
\hfill
\subfloat[\textbf{PAMAP2 (subjects 101+105):} activity-level window coverage.]{%
  \includegraphics[width=0.49\linewidth]{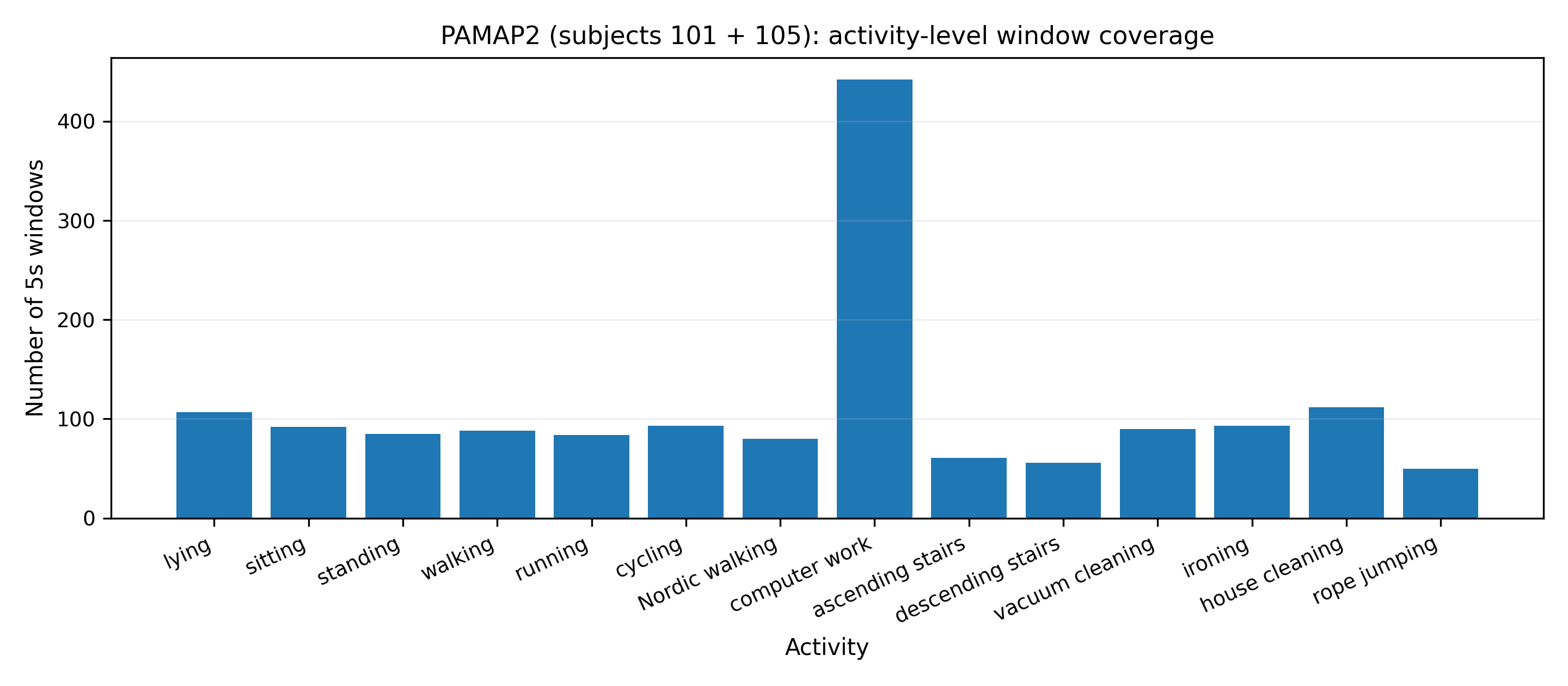}}%
\caption{Activity-level window counts (5\,s windows) in (a) the in-house dataset and (b) the PAMAP2 benchmark (subjects 101+105). PAMAP2 activity mapping follows the dataset documentation~\cite{reiss2012iswc,reiss2012abra}.}
\label{fig:fig1}
\end{figure}

In both datasets, a small number of activities accounted for a large fraction of available windows, while several activities were represented by comparatively few windows (Fig.~\ref{fig:fig1}).

\subsection{Accuracy estimates showed uncertainty and placement dependence}
Per-class accuracy varied across activities and was associated with non-negligible statistical uncertainty (Fig.~\ref{fig:fig2}).

\begin{figure}[t]
\centering
\subfloat[\textbf{In-house HAR:} per-class accuracy across form factors with 95\% CIs (5 tests/class).]{%
  \includegraphics[width=0.49\linewidth]{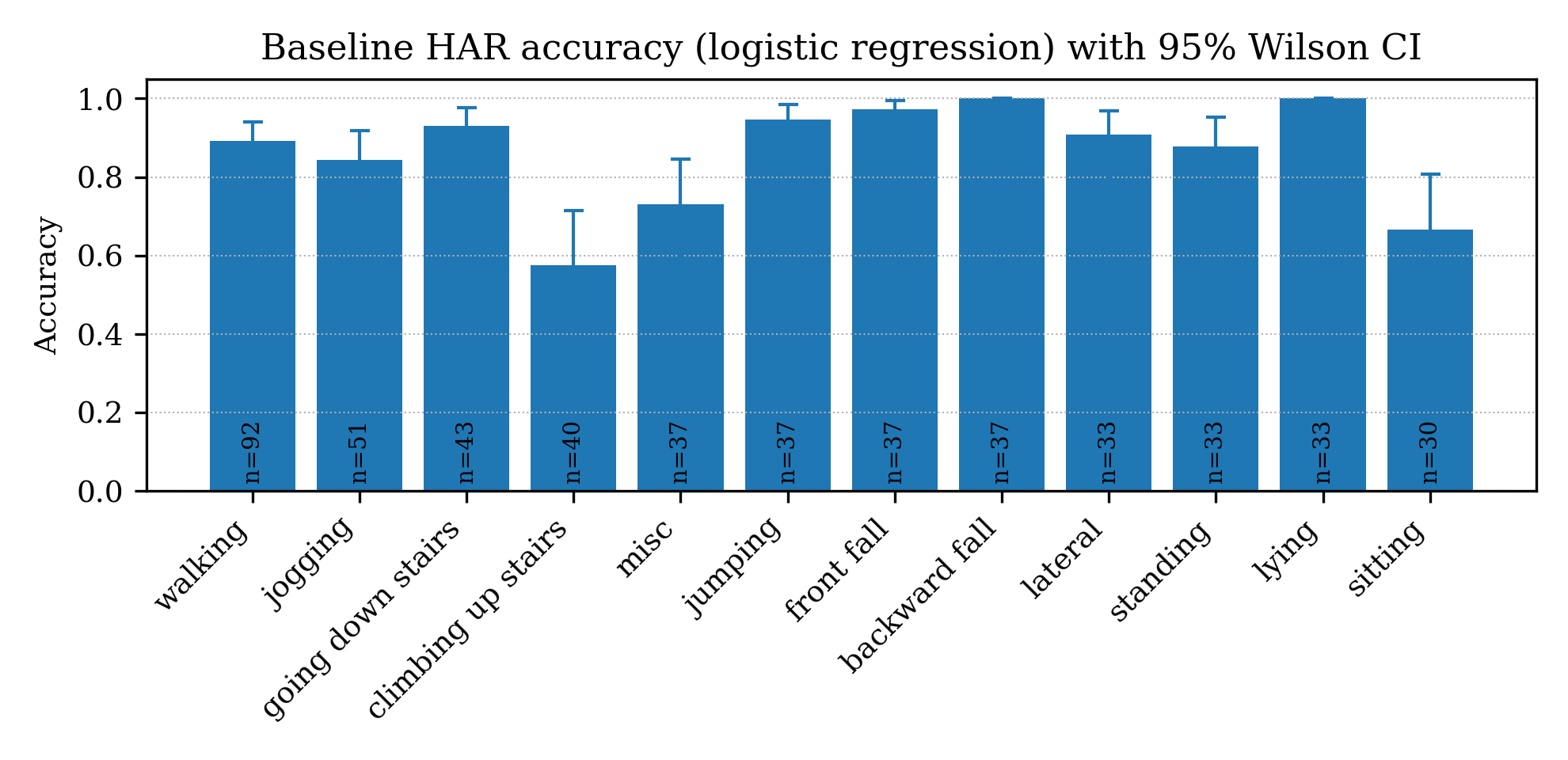}}%
\hfill
\subfloat[\textbf{PAMAP2:} per-class accuracy across placements (hand/chest/ankle) with Wilson 95\% CIs.]{%
  \includegraphics[width=0.49\linewidth]{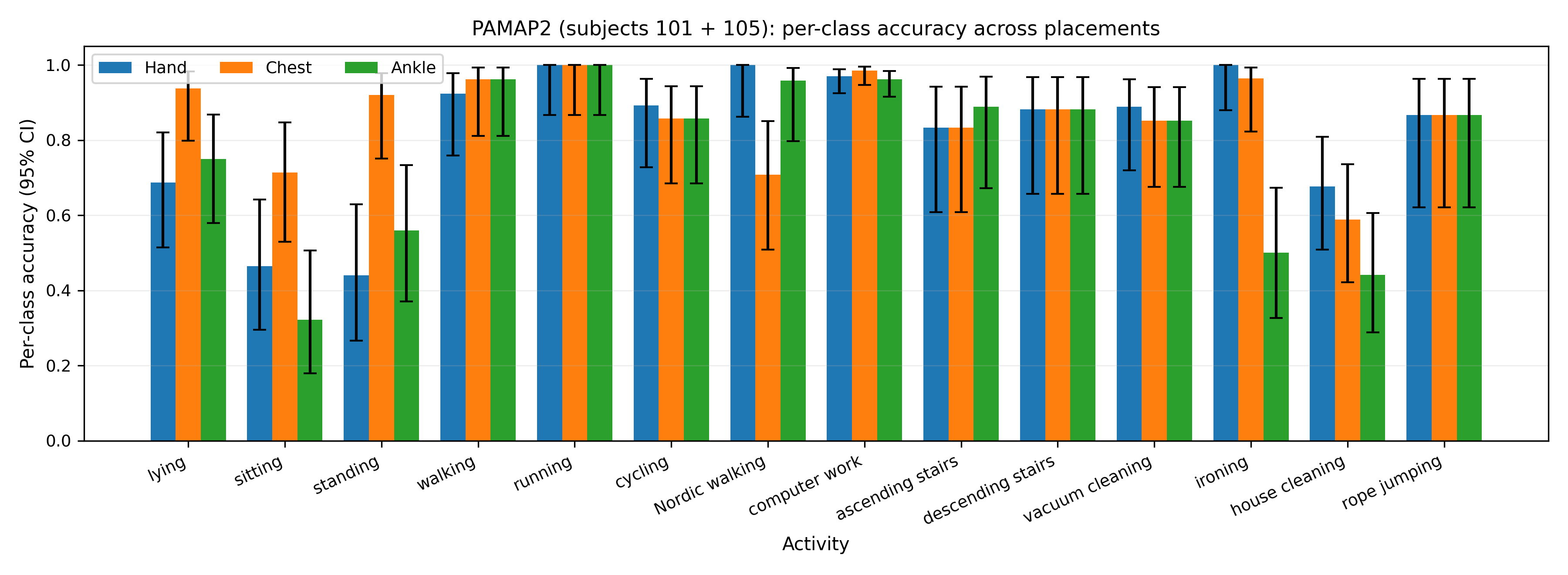}}%
\caption{Per-class accuracy with 95\% confidence intervals. (a) In-house dataset (5 test windows/class). (b) PAMAP2 dataset: Wilson 95\% intervals; per-class test support ranges from 15 to 133 windows (median 26.5) across hand, chest, and ankle placements.}
\label{fig:fig2}
\end{figure}

In the in-house dataset, confidence intervals remained wide for several classes due to the small per-class test support ($N=5$ windows/class). In PAMAP2, accuracy differed across sensor placements (hand, chest, ankle) for multiple activities.

\subsection{Blind-spot mass increased sharply with support threshold}
Blind-spot mass increased sharply as the required per-state support threshold $\tau$ increased (Fig.~\ref{fig:fig3}).

\begin{figure}[t]
\centering
\subfloat[\textbf{In-house HAR:} $\widehat{B}_n(\tau)$ vs $\tau$ for activity-only states.]{%
  \includegraphics[width=0.49\linewidth]{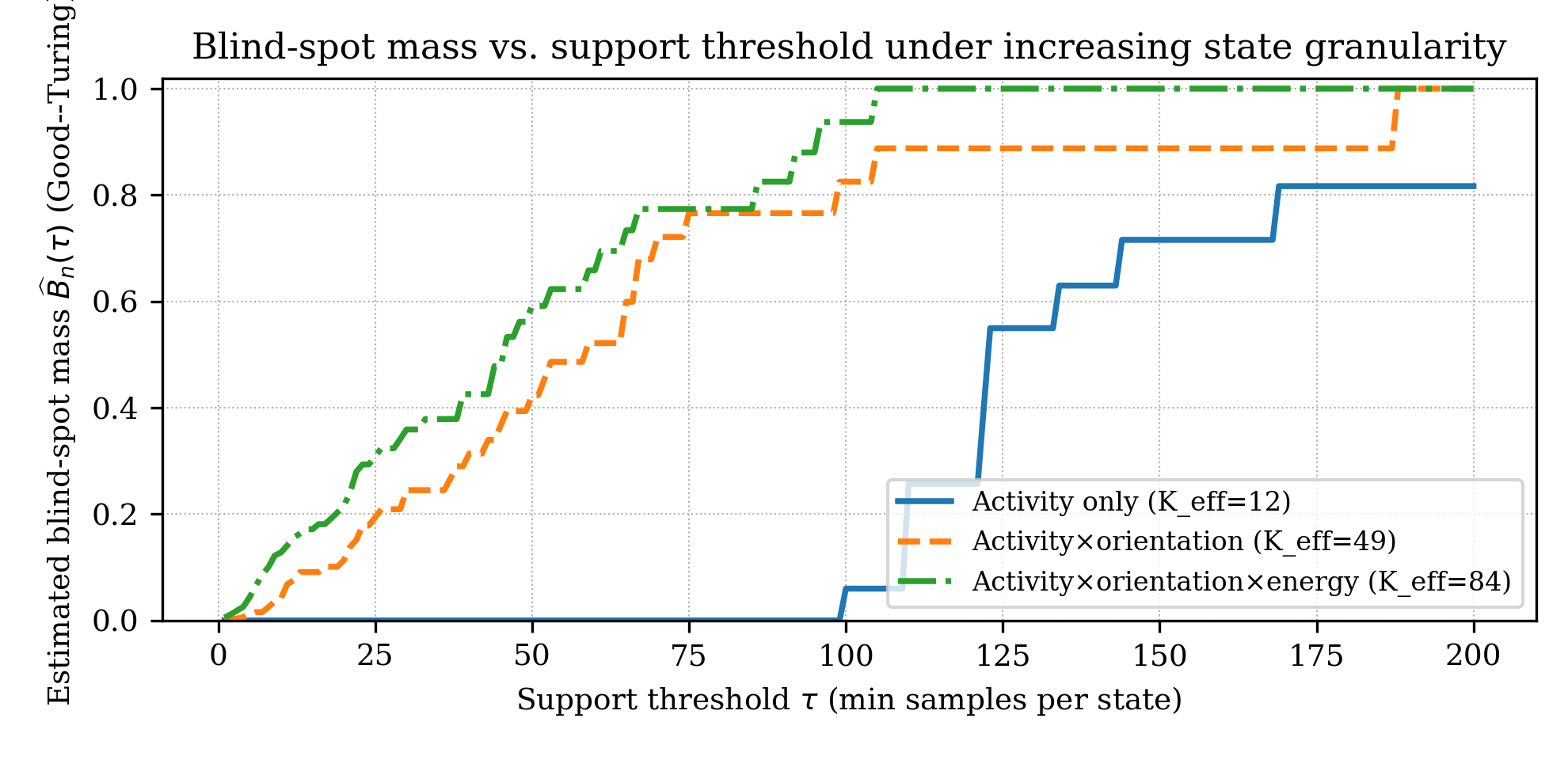}}%
\hfill
\subfloat[\textbf{PAMAP2:} $\widehat{B}_n(\tau)$ vs $\tau$ under refined state abstractions.]{%
  \includegraphics[width=0.49\linewidth]{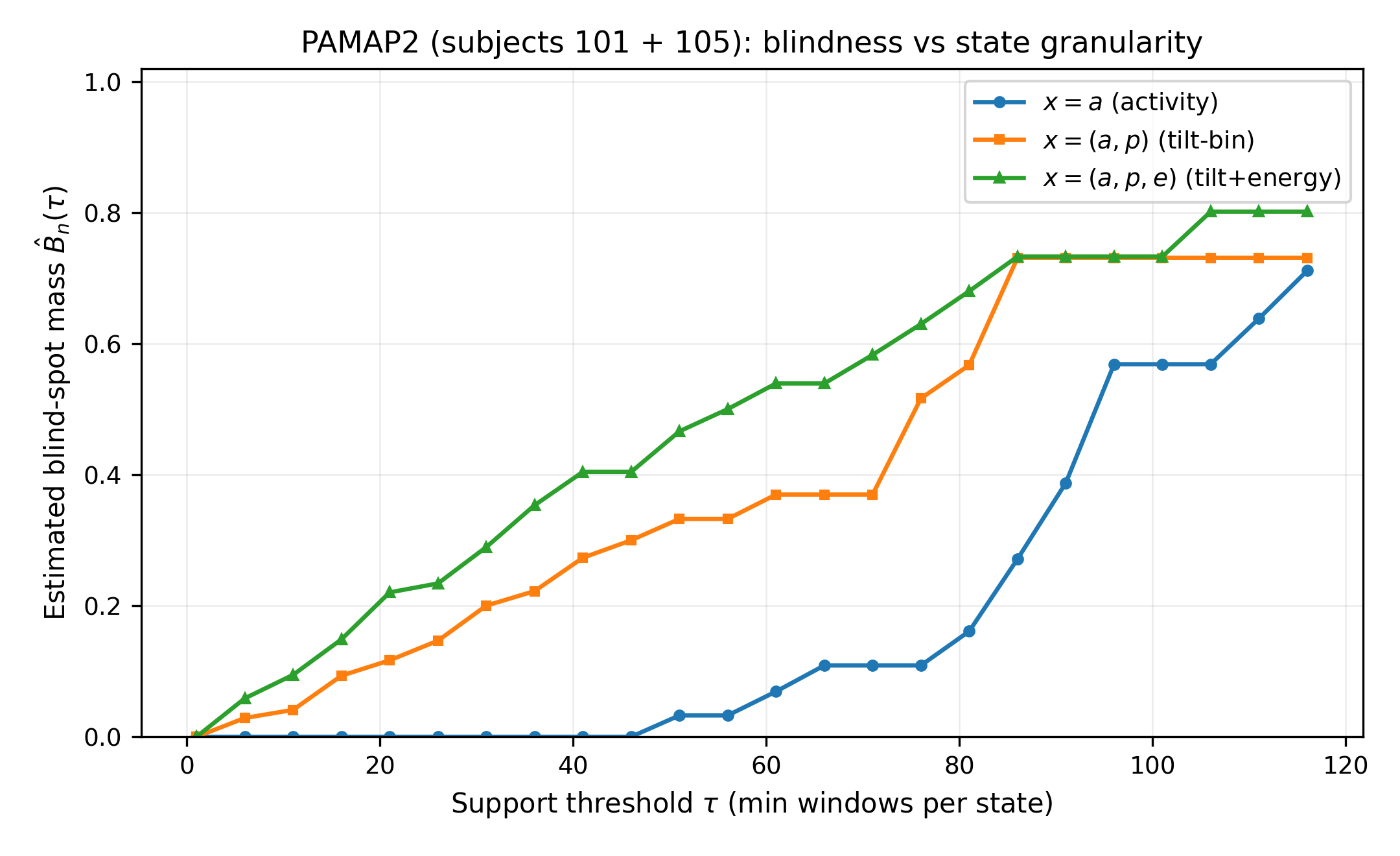}}%
\caption{Estimated blind-spot mass $\widehat{B}_n(\tau)$ versus support threshold $\tau$. (a) In-house dataset under activity-only states $x=a$. (b) PAMAP2 under refined operational abstractions: $x=a$ ($K_{\mathrm{eff}}=14$), $x=(a,p)$ ($K_{\mathrm{eff}}=44$), and $x=(a,p,e)$ ($K_{\mathrm{eff}}=78$).}
\label{fig:fig3}
\end{figure}

For PAMAP2, refined operational abstractions yielded larger blind-spot mass over the same range of $\tau$ compared with activity-only states.

\subsection{Cross-domain replication on MIMIC-IV clinical state abstractions}
We performed a cross-domain validation on the MIMIC-IV Clinical Database Demo v2.2 (275 admissions) by treating each admission as one sample from a discrete clinical state space (Methods). 
The resulting blind-spot mass curve closely matched the deployment-refined HAR curve (Fig.~\ref{fig:crossdomain}). 
Specifically, both domains converged to $\widehat{B}_n(5)\approx 0.95$ (HAR: 0.949; MIMIC-IV: 0.949), indicating that under practical support requirements, the majority of deployment probability mass resides in under-supported states.

\begin{figure}[t]
\centering
\includegraphics[width=0.75\linewidth]{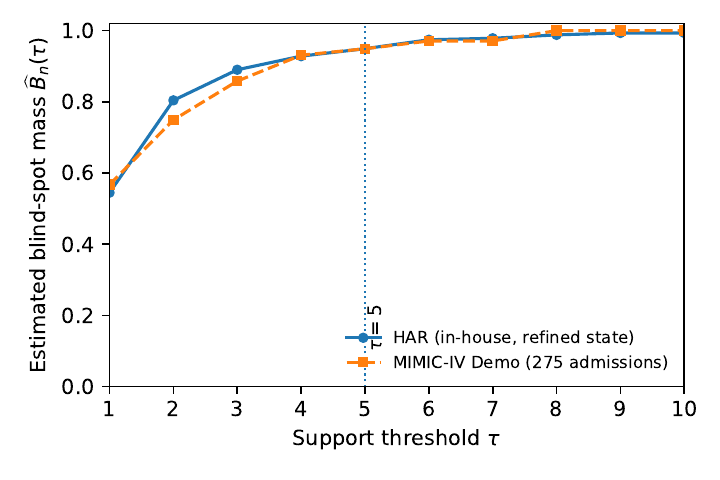}
\caption{Cross-domain replication of blind-spot mass curves. We compare a deployment-refined operational abstraction for in-house HAR windows against an admission-level clinical state abstraction on the MIMIC-IV Demo dataset (275 admissions). Both curves converge to $\widehat{B}_n(5)\approx 0.95$, illustrating that coverage blindness arises from heavy-tailed, combinatorial state spaces across domains.}
\label{fig:crossdomain}
\end{figure}

\subsection{Coverage-imposed accuracy ceiling decreased with support requirements}
The coverage-imposed accuracy ceiling decreased as $\tau$ increased and tracked the growth of blind-spot mass (Fig.~\ref{fig:fig4}).

\begin{figure}[t]
\centering
\subfloat[\textbf{In-house HAR:} accuracy ceiling vs $\tau$ overlaid with $\widehat{B}_n(\tau)$.]{%
  \includegraphics[width=0.49\linewidth]{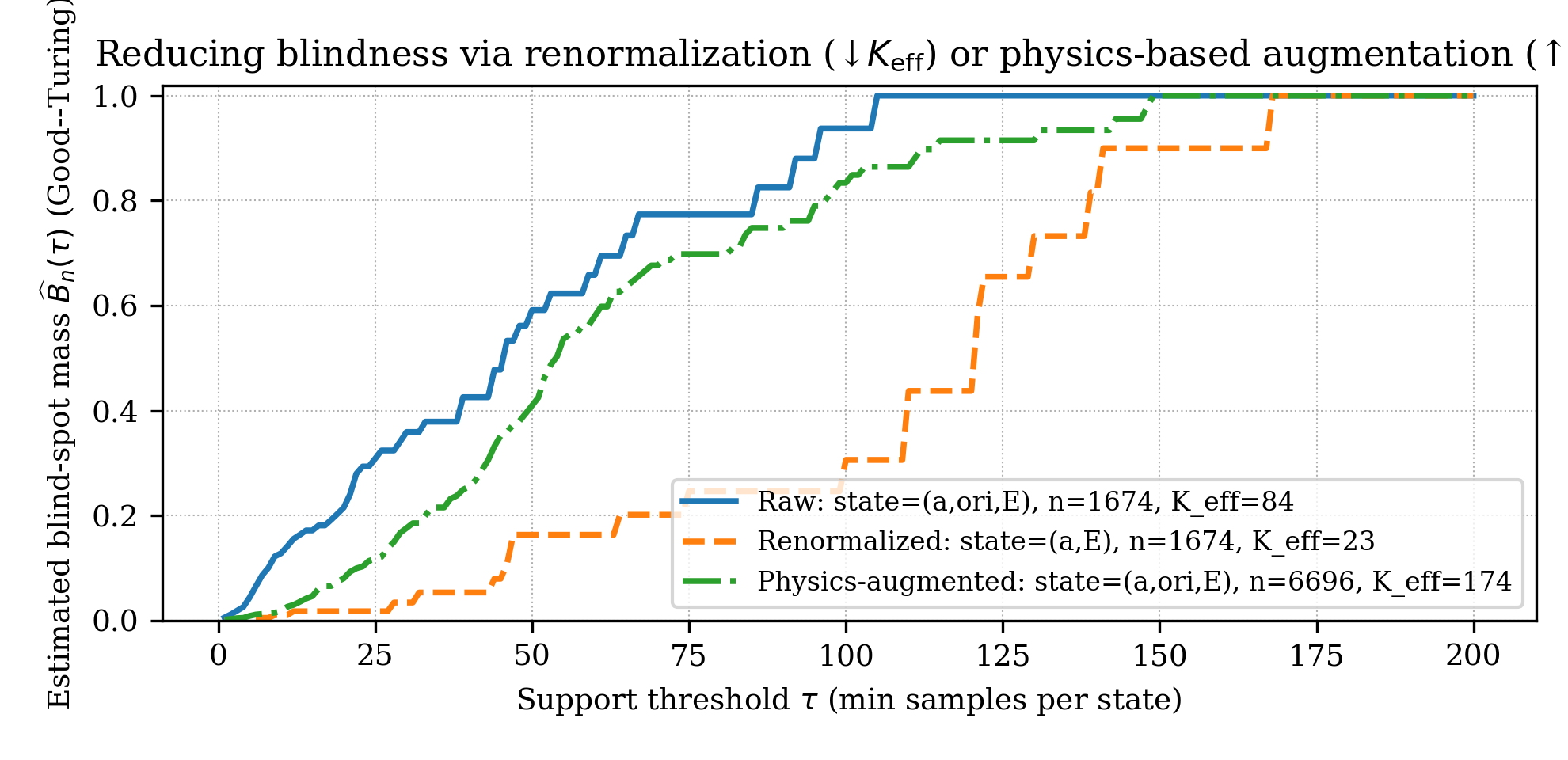}}%
\hfill
\subfloat[\textbf{PAMAP2:} accuracy ceiling vs $\tau$ overlaid with $\widehat{B}_n(\tau)$ under refined abstraction.]{%
  \includegraphics[width=0.49\linewidth]{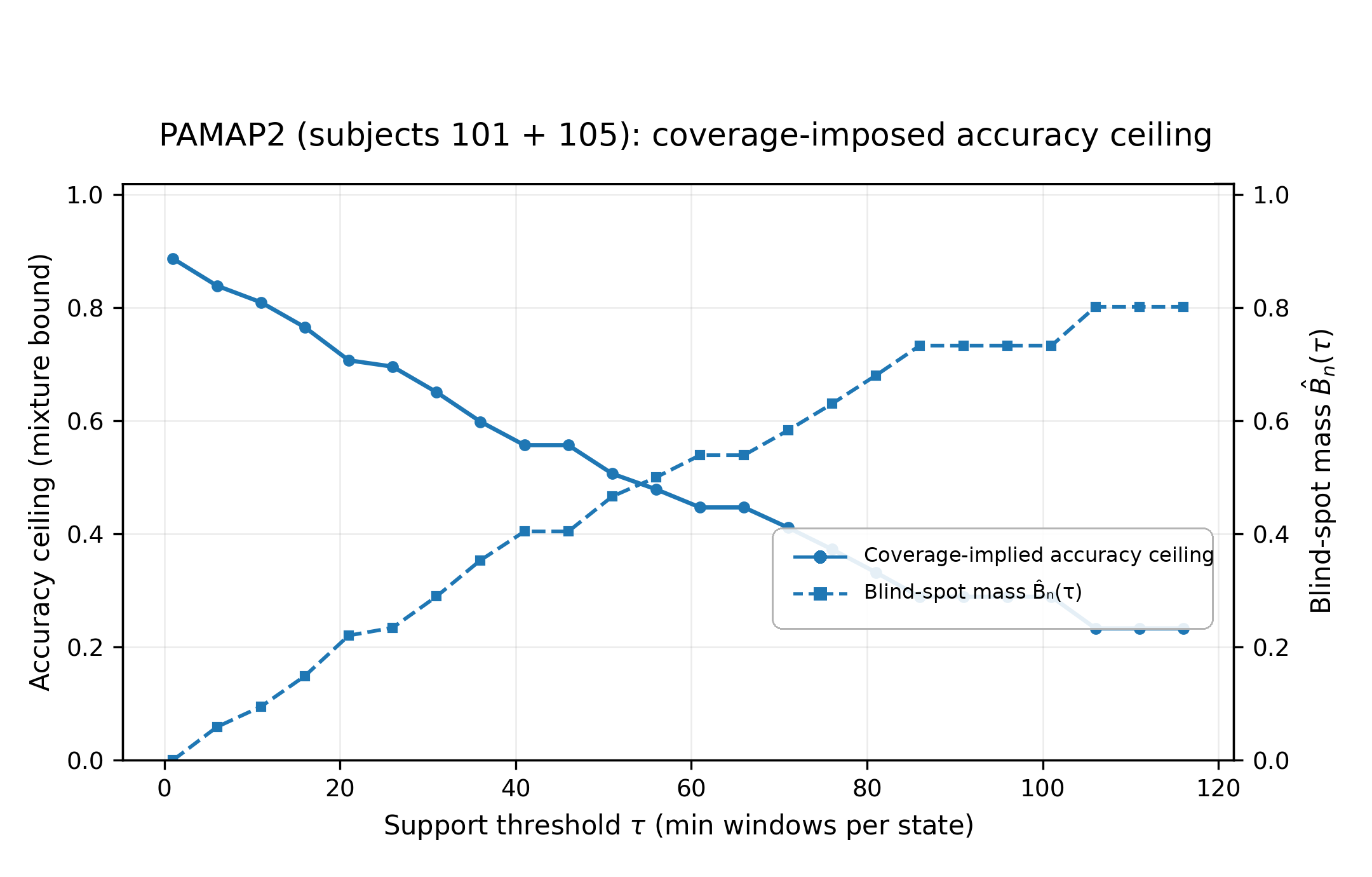}}%
\caption{Coverage-imposed accuracy ceiling versus $\tau$, overlaid with $\widehat{B}_n(\tau)$.}
\label{fig:fig4}
\end{figure}

Across both datasets, the ceiling declined monotonically with $\tau$ and declined more steeply under refined abstractions.

\subsection{Blindness contributions were concentrated in specific regimes}
Blind-spot mass contributions were concentrated among a subset of activities (in-house) or composite regimes (PAMAP2) at fixed $\tau$ (Fig.~\ref{fig:fig5}).

\begin{figure}[t]
\centering
\subfloat[\textbf{In-house HAR:} per-activity contributions to blind mass at $\tau=150$.]{%
  \includegraphics[width=0.49\linewidth]{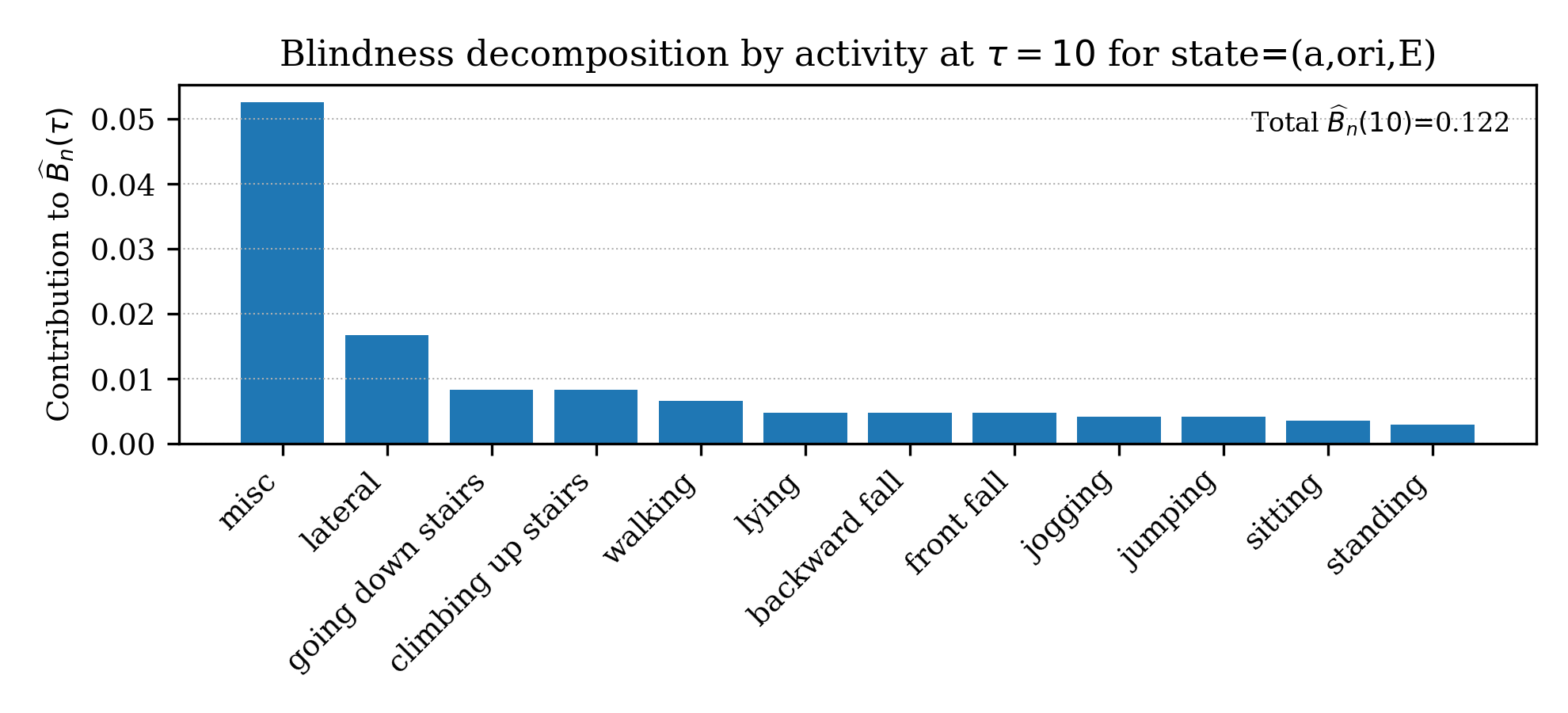}}%
\hfill
\subfloat[\textbf{PAMAP2:} top composite-state contributions at $\tau=20$ (activity, tilt-bin, energy-bin).]{%
  \includegraphics[width=0.49\linewidth]{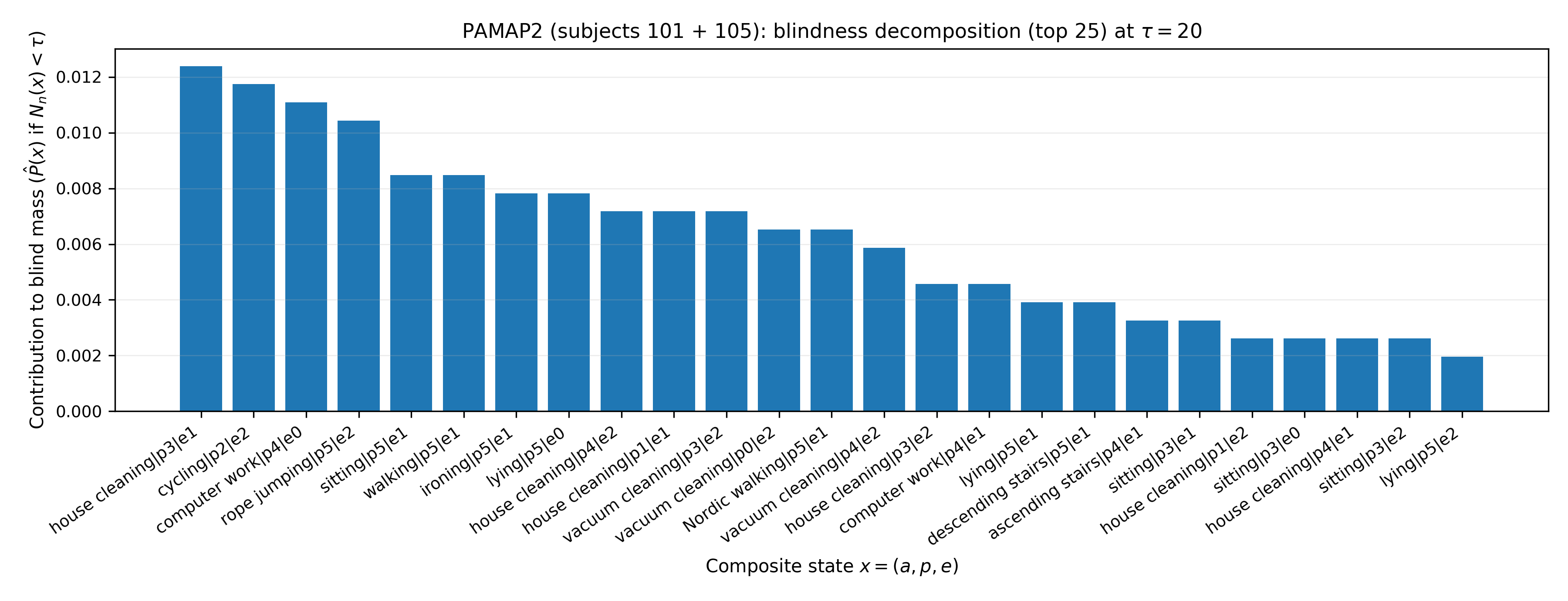}}%
\caption{Decomposition of blind-spot mass at fixed $\tau$. (a) In-house dataset: per-activity contributions at $\tau=150$. (b) PAMAP2: top composite-state contributions at $\tau=20$ for $x=(a,p,e)$ (activity $a$, tilt-bin $p$, energy-bin $e$).}
\label{fig:fig5}
\end{figure}

\begin{table}[t]
\caption{Illustrative risk-weighted blindness example (in-house dataset; $\tau=150$). Risk weights are illustrative and domain-dependent.}
\label{tab:risk_blindness}
\centering
\footnotesize
\setlength{\tabcolsep}{5pt}
\begin{tabular}{lrrrr}
\toprule
State $x$ & $N_n(x)$ & $\hat{P}(x)$ & $w(x)$ & Contribution\\
\midrule
Walking & 307 & 0.183 & 0.2 & 0.000\\
Stairs up & 134 & 0.080 & 0.6 & 0.048\\
Stairs down & 144 & 0.086 & 0.6 & 0.052\\
Front fall & 122 & 0.073 & 1.0 & 0.073\\
Backward fall & 122 & 0.073 & 1.0 & 0.073\\
\bottomrule
\end{tabular}
\end{table}

Table~\ref{tab:risk_blindness} reports an illustrative risk-weighted calculation for the in-house dataset at $\tau=150$ under the weights listed in the table.

\section{Discussion and Deployment Implications}\label{sec:discussion}
The theory and cross-dataset results jointly indicate that, as wearable state definitions approach the true $\Omega$ in~\eqref{eq:wearable_state}, blind mass increases for any fixed $n$. The PAMAP2 results make this explicit: refining the state abstraction from activity-only $x=a$ to $x=(a,p,e)$ (by introducing measurable proxies for placement/orientation and intensity) increases $\widehat{B}_n(\tau)$ for the same dataset size. This is precisely the effective state-space phenomenon in~\eqref{eq:keff}: a richer state description increases $K_{\mathrm{eff}}$, which increases the number of rare states and therefore the blind mass in~\eqref{eq:bn}.

\textbf{Quantifying required normalization and physics-informed constraints.} The blindness curve $\widehat{B}_n(\tau)$ provides a quantitative answer: choose $\tau$ to represent the minimum support required for the desired reliability (e.g., robustness across users/placements), and measure the resulting blind mass. If $\widehat{B}_n(\tau)$ is large (Fig.~\ref{fig:fig3}), then either (i) more data must be collected in the blind region, or (ii) the effective state space must be reduced. Figure~\ref{fig:fig5} then specifies \emph{where} to focus: it decomposes the blind mass into the states (or composite regimes) that dominate the missing probability mass.

\textbf{Impact of physics-informed constraints on state-space reduction.} Purely data-driven collection scales poorly with $|\Omega|$. Renormalization collapses nuisance degrees of freedom so that many raw conditions map to the same effective state, directly reducing $K_{\mathrm{eff}}$. For example, cadence normalization enforces an approximate invariance to gait frequency; under inverted-pendulum gait dynamics, stride acceleration profiles can be rescaled without increasing $K_{\mathrm{eff}}$. Physics-based modeling and constrained generative augmentation further reduce sampling burden by enforcing invariances and physically plausible trajectories, shifting probability mass from blind regions into supported regimes. In edge wearables, this is particularly important because on-device models are intentionally small; reliability must come from structured state representations and coverage-aware dataset design, not solely from larger networks.

The cross-domain replication on MIMIC-IV (Section~5.4) provides further evidence that coverage blindness is not a wearable-specific phenomenon. HAR windows and hospital admissions differ fundamentally in modality, feature space, label space, and application domain, yet both produce heavy-tailed state distributions that converge to $\widehat{B}_n(5) \approx 0.95$ under practical support requirements. This structural similarity suggests that the combinatorial growth of operational state spaces --- and the resulting long-tail under-support at finite $n$ --- is a general property of ML deployment rather than an artifact of sensor data or activity recognition. Practitioners in domains beyond wearables should therefore expect comparable blind-spot mass curves when operational state definitions are refined to reflect true deployment variation.

We formulated a coverage-blindness evaluation methodology for wearable edge AI and validated it on both an in-house dataset and the open PAMAP2 benchmark (Fig.~1--5). The paired results show that high accuracy can coexist with large blind-spot mass; deployable accuracy is bounded by coverage via Proposition~\ref{prop:ceiling}. The blindness measure and its decomposition provide a principled way to decide when data renormalization and physics-based modeling are required, and how to prioritize data collection and modeling effort for reliable field deployment under combinatorial variation.

\section{Conclusion}
Blind-spot mass $B_n(\tau)$ provides a simple, deployment-facing quantification of coverage risk in ML systems operating over large operational state spaces.
By linking under-support to an explicit accuracy decomposition, the framework separates model limitations from data coverage limitations and supplies actionable diagnostics through state refinement and blindness decomposition.
Our wearable HAR validation on an in-house dataset and PAMAP2 illustrates how coverage risk can grow sharply as operational state definitions become more realistic, motivating coverage-aware data collection, renormalization, and physics-informed constraints for reliable edge deployment.
Future work will extend blind-spot mass evaluation to additional application domains and to adaptive deployment monitoring pipelines.
Cross-domain replication on clinical admission data further demonstrates that coverage blindness is a structural property of finite ML deployment, not an artifact of any single application domain.

\section{Data and Code Availability}
The PAMAP2 Physical Activity Monitoring dataset is publicly available via the UCI Machine Learning Repository (\url{https://archive.ics.uci.edu/dataset/231}). The in-house wearable inertial sensor dataset generated during this study has been deposited in Zenodo under open access and is freely available at \url{https://doi.org/10.5281/zenodo.18480659} under a Creative Commons Attribution 4.0 International License.

All analysis scripts used to compute the coverage blindness metrics and generate the figures in this manuscript
are provided with this submission as Supplementary Software and will be released under an open-source license upon publication.

\acks{Biplab Pal acknowledges Prof.\ Aryaa Gangopadhyay, Director of CARDS at UMBC, for guidance and support.
Santanu Bhattacharya acknowledges Prof.\ Ramesh Raskar and Ayush Chopra of the MIT Media Lab for helpful suggestions toward this work.
\textbf{Funding:} This study received no specific external funding.
\textbf{Competing interests:} The authors declare no competing interests.}

\end{document}